\documentclass[oneside,reqno,english]{amsart}
\usepackage[T1]{fontenc}
\usepackage[latin9]{inputenc}
\usepackage{babel}
\usepackage{refstyle}
\usepackage{float}
\usepackage{mathrsfs}
\usepackage{amstext}
\usepackage{amsthm}
\usepackage{amssymb}
\usepackage{graphicx}
\usepackage[all]{xy}
\usepackage[unicode=true,pdfusetitle,
 bookmarks=true,bookmarksnumbered=false,bookmarksopen=false,
 breaklinks=false,pdfborder={0 0 0},pdfborderstyle={},backref=false,colorlinks=false]
 {hyperref}
\hypersetup{
 colorlinks=true,citecolor=blue,linkcolor=blue,linktocpage=true}

\makeatletter

\AtBeginDocument{\providecommand\secref[1]{\ref{sec:#1}}}
\AtBeginDocument{\providecommand\lemref[1]{\ref{lem:#1}}}
\AtBeginDocument{\providecommand\corref[1]{\ref{cor:#1}}}
\AtBeginDocument{\providecommand\thmref[1]{\ref{thm:#1}}}
\AtBeginDocument{\providecommand\defref[1]{\ref{def:#1}}}
\AtBeginDocument{\providecommand\figref[1]{\ref{fig:#1}}}
\AtBeginDocument{\providecommand\exaref[1]{\ref{exa:#1}}}
\AtBeginDocument{\providecommand\tabref[1]{\ref{tab:#1}}}
\providecommand{\tabularnewline}{\\}
\RS@ifundefined{subsecref}
  {\newref{subsec}{name = \RSsectxt}}
  {}
\RS@ifundefined{thmref}
  {\def\RSthmtxt{theorem~}\newref{thm}{name = \RSthmtxt}}
  {}
\RS@ifundefined{lemref}
  {\def\RSlemtxt{lemma~}\newref{lem}{name = \RSlemtxt}}
  {}

\numberwithin{equation}{section}
\numberwithin{figure}{section}
\theoremstyle{plain}
\newtheorem{thm}{\protect\theoremname}[section]
\theoremstyle{plain}
\newtheorem{question}[thm]{\protect\questionname}
\theoremstyle{plain}
\newtheorem{lem}[thm]{\protect\lemmaname}
\theoremstyle{definition}
\newtheorem{defn}[thm]{\protect\definitionname}
\theoremstyle{remark}
\newtheorem{rem}[thm]{\protect\remarkname}
\theoremstyle{definition}
\newtheorem{example}[thm]{\protect\examplename}
\theoremstyle{plain}
\newtheorem{cor}[thm]{\protect\corollaryname}

\providecommand{\MR}[1]{}

\usepackage{bbm}
\allowdisplaybreaks
\usepackage{needspace}
\usepackage{refstyle}
\usepackage{enumitem}

\setlist[enumerate]{itemsep=5pt,topsep=3pt}
\setlist[enumerate,1]{label=\textup{(}\roman*\textup{)},ref=\roman*}
\setlist[enumerate,2]{label=\textup{(}\alph*\textup{)},ref=\theenumi \alph*}

\newref{lem}{refcmd={Lemma \ref{#1}}}
\newref{thm}{refcmd={Theorem \ref{#1}}}
\newref{cor}{refcmd={Corollary \ref{#1}}}
\newref{sec}{refcmd={Section \ref{#1}}}
\newref{subsec}{refcmd={Section \ref{#1}}}
\newref{chap}{refcmd={Chapter \ref{#1}}}
\newref{prop}{refcmd={Proposition \ref{#1}}}
\newref{exa}{refcmd={Example \ref{#1}}}
\newref{tab}{refcmd={Table \ref{#1}}}
\newref{rem}{refcmd={Remark \ref{#1}}}
\newref{def}{refcmd={Definition \ref{#1}}}
\newref{fig}{refcmd={Figure \ref{#1}}}
\newref{claim}{refcmd={Claim \ref{#1}}}

\makeatother

\providecommand{\corollaryname}{Corollary}
\providecommand{\definitionname}{Definition}
\providecommand{\examplename}{Example}
\providecommand{\lemmaname}{Lemma}
\providecommand{\questionname}{Question}
\providecommand{\remarkname}{Remark}
\providecommand{\theoremname}{Theorem}

\begin{document}
\title[optimal feature selection via p.d. kernels]{Conditional mean embeddings and optimal feature selection via positive
definite kernels}
\begin{abstract}
Motivated by applications, we consider here new operator theoretic
approaches to Conditional mean embeddings (CME). Our present results
combine a spectral analysis-based optimization scheme with the use
of kernels, stochastic processes, and constructive learning algorithms.
For initially given non-linear data, we consider optimization-based
feature selections. This entails the use of convex sets of positive
definite (p.d.) kernels in a construction of optimal feature selection
via regression algorithms from learning models. Thus, with initial
inputs of training data (for a suitable learning algorithm,) each
choice of p.d. kernel $K$ in turn yields a variety of Hilbert spaces
and realizations of features. A novel idea here is that we shall allow
an optimization over selected sets of kernels $K$ from a convex set
$C$ of positive definite kernels $K$. Hence our \textquotedblleft optimal\textquotedblright{}
choices of feature representations will depend on a secondary optimization
over p.d. kernels $K$ within a specified convex set $C$.
\end{abstract}

\author{Palle E.T. Jorgensen}
\address{(Palle E.T. Jorgensen) Department of Mathematics, The University of
Iowa, Iowa City, IA 52242-1419, U.S.A.}
\email{palle-jorgensen@uiowa.edu}
\author{Myung-Sin Song}
\address{(Myung-Sin Song) Department of Mathematics and Statistics, Southern
Illinois University Edwardsville, Edwardsville, IL 62026, USA}
\email{msong@siue.edu}
\author{James Tian}
\address{(James F. Tian) Mathematical Reviews, 416 4th Street Ann Arbor, MI
48103-4816, U.S.A.}
\email{jft@ams.org}

\maketitle
\tableofcontents{}

\section{Introduction}

Recently the mathematical tools for what is often called \emph{Conditional
mean embeddings} (CME) have played a role in multiple and new applications
\cite{MR4121886,MR4242933,10.5555/3042573.3042803,park2021measuretheoretic,pmlr-v124-ray-chowdhury20a,aaaiLeverSSS16,MR4187269}.
One reason for this is that they (the CMEs) stand at the crossroads
of stochastic processes and constructive learning algorithms. Our
present focus will be a new use of CMEs in an analysis of optimization-based
selections of positive definite (p.d.) kernels (and their associated
reproducing kernel Hilbert spaces RKHS), and their use in a construction
of optimal feature selection via regression algorithms for particular
learning models; see \cite{MR4288278,MR4282408,MR4255285}. Our present
use of \emph{positive definite kernels} $K$, defined on $X\times X$,
serves two purposes: First, every positive definite kernel $K$ is
a covariance kernel for a centered Gaussian process indexed by $X$,
so in particular there are associated \emph{probability spaces} realized
in a generalized path space, with sigma-algebra, and probability measures
$\mathbb{P}$. Secondly, every choice of a p.d. kernel $K$ yields
factorizations via Hilbert space, and so each choice of $K$ opens
up a variety choices of Hilbert spaces allowing in turn realization
of features as they are reflected in initial inputs of training data
(for a suitable learning algorithm.) In earlier approaches to such
generalized regression analyses, the p.d. kernel $K$ for the model
is given at the outset. By contrast, a novel idea in our present approach
to selection of features is that we shall allow an optimization over
suitably selected sets of kernels $K$ in a convex set $C$ of positive
definite kernels $K$. Hence our \textquotedblleft optimal\textquotedblright{}
choices of feature representations will depend on a secondary optimization
over kernels $K$ within a specified convex set $C$. Below, we begin
with a summary of the mathematical notions which will enter our analysis,
starting with the tools we need from Conditional mean embeddings (CME).

Our present approach to feature selection is motivated in part by
machine learning and data mining. Such uses are typically dictated
by \textquotedblleft big data,\textquotedblright{} and the need for
dimension reduction. This refers to the process of transforming the
data from the high-dimensional space into a space of fewer dimensions,
such as to avoid loss of \textquotedblleft essential\textquotedblright{}
information \cite{MR2810909,MR2558684,MR2228737}. The linear case
of data transformation encompasses principal component analysis (PCA),
while by contrast, the nonlinear theories make use of kernel theory,
our present focus. In our approach we aim for adaptive selections
of nonlinear mappings serving to maximize the variance in the data,
hence the design of optimal kernels for the task at hand. Such approaches
are especially useful for dealing with clustering, and with the need
for selection of partitions of the total data-set into some natural
connected components. For kernel learning, we refer to \cite{MR4349381,MR4341244,MR4329775,MR4318510}.

\section{\label{sec:OW}Overview}

In the discussion below, we shall make use of some facts from the
analysis and geometry which arise naturally from the use of positive
definite kernels $K$, selection of features via factorization, and
the use of reproducing kernel Hilbert spaces $\mathscr{H}_{K}$ for
regression and optimization. While this list of topics is well covered
in the literature, the references \cite{MR4295177,MR4218424,MR4106884,MR4020693,MR3964760}
are especially relevant for what we need.

In summary, the purpose of our paper is illustrated with the following
framework. Problem: Selection of optimal positive definite (p.d.)
kernels $K$ for use in feature analysis, adapted to large training
data: 
\[
\varphi\rightsquigarrow K\rightsquigarrow f\quad\quad\left.\begin{alignedat}{1}\varphi\:\text{training data}\\
f\:\text{feature selection}
\end{alignedat}
\right\} \;\text{depends on choices of p.d. \ensuremath{K}.}
\]

\begin{question}
~
\begin{enumerate}
\item What is the best $f$ in $\varphi\rightsquigarrow\left(\mu,K\right)\rightsquigarrow f$?
Here, $\mu$ is a fixed measure on $X$ and $K:X\times X\rightarrow\mathbb{C}$
is p.d. 
\item What is the best $K$ when $\varphi$ and $\mu$ are fixed? How to
adjust $K$ to optimize $f$? 
\end{enumerate}
\end{question}

The following diagram shows a workflow for learning training data
via choices of p.d. kernels, which returns an optimal feature. 

\[
\xymatrix{\text{\ensuremath{\begin{matrix}\text{learning}\\
 \text{training data}\\
 \text{\ensuremath{\varphi\in L^{2}\left(\mu\right)}} 
\end{matrix}}}\ar[rd]\ar[rr] &  & \text{\ensuremath{\begin{matrix}\text{feature \ensuremath{f^{\varphi,K}} in}\\
\text{\ensuremath{\mathscr{H}_{K}(\text{RKHS})}}
\end{matrix}}}\\
 & \text{choices of }K\ar[ru]
}
\]

The ``best'' kernel $K$ is the one that picks out the best features
for a given pair $\left(\varphi,\mu\right)$: 
\begin{lem}
\label{lem:B2}We have
\begin{align}
f^{\varphi,K} & =\mathop{\text{argmin}\left\{ \left\Vert \varphi-T_{\mu}f\right\Vert _{L^{2}\left(\mu\right)}^{2}+\alpha\left\Vert f\right\Vert _{\mathscr{H}_{K}}^{2}:f\in\mathscr{H}_{K}\right\} }\label{eq:a1}\\
 & =T_{\mu}^{*}\left(\alpha I+T_{\mu}T_{\mu}^{*}\right)^{-1}\varphi.\label{eq:a2}
\end{align}
\end{lem}

\begin{proof}
This is a well-known result. See, e.g., \cite{MR3564937,MR3547633,10.5555/3042573.3042803,jorgensen2021positive},
and \secref{OA}. 
\end{proof}
The solution in (\ref{eq:a1})--(\ref{eq:a2}) depends directly on
$K$, and picking ``best'' features will mean $\left\Vert f^{\varphi,K}\right\Vert _{\mathscr{H}_{K}}^{2}$
at a maximum. 

One may fix $\varphi$, and optimize on choices of $\left(K,\mu\right)$,
that is, 
\begin{equation}
\max_{K,\mu}\left\Vert f^{\varphi,K}\right\Vert _{\mathscr{H}_{K}}^{2}.
\end{equation}

Alternatively, fix $\mu$, and optimize on a convex set of kernels
$K\in\mathscr{K}\left(\mu\right)$, see definition below. 
\begin{defn}
\label{def:b3}Given a set $X$ with measure $\mu$, a pair $\left(K,\mu\right)$
is said to be \emph{admissible} if 
\begin{equation}
\mathscr{H}_{K}\ni K\left(\cdot,y\right)\xrightarrow{\quad T_{\mu,K}\quad}K\left(\cdot,y\right)\in L^{2}\left(\mu\right),\label{eq:a4}
\end{equation}
extended by linearity, is well defined and closable.
\end{defn}

\begin{defn}
\label{def:K}Fix $\mu$, let 
\begin{align}
\mathscr{K}\left(\mu\right) & =\left\{ K:\left(K,\mu\right)\:\text{is admissible}\right\} \label{eq:b5}\\
\mathscr{K}_{b}\left(\mu\right) & =\left\{ K:\left(K,\mu\right)\:\text{is admissible and \ensuremath{T_{K,\mu}} is bounded}\right\} .
\end{align}
Further, for a fixed $K$, let 
\begin{equation}
\mathfrak{M}\left(K\right)=\left\{ \mu:\left(K,\mu\right)\:\text{is admissible}\right\} .
\end{equation}
\end{defn}

\begin{lem}
\label{lem:b5}Suppose $\left(K,\mu\right)$ is admissible. Then the
adjoint operator $T_{\mu,K}^{*}:L^{2}\left(\mu\right)\rightarrow\mathscr{H}_{K}$
is as follows: 
\begin{equation}
\left(T_{\mu,K}^{*}f\right)\left(\cdot\right)=\int K\left(\cdot,y\right)f\left(y\right)\mu\left(dy\right).\label{eq:a5}
\end{equation}
\end{lem}

\begin{proof}
To verify (\ref{eq:a5}), we must show that 
\begin{align}
\left\langle K\left(\cdot,x\right),T_{\mu,K}^{*}f\right\rangle _{\mathscr{H}_{K}} & =\left\langle T_{\mu,K}K\left(\cdot,x\right),f\right\rangle _{L^{2}\left(\mu\right)}=\left\langle K\left(\cdot,x\right),f\right\rangle _{L^{2}\left(\mu\right)}\label{eq:a6}
\end{align}
But 
\[
\text{LHS}_{\left(\ref{eq:a6}\right)}=\int K\left(x,y\right)f\left(y\right)\mu\left(dy\right)=\text{RHS}_{\left(\ref{eq:a6}\right)},
\]
by the definition of $T_{\mu,K}^{*}$ and the reproducing property
of $\mathscr{H}_{K}$.
\end{proof}
To see that $T_{\mu,K}^{*}$ is well defined, we shall need the following
technical lemma.
\begin{lem}
If $\mu$ is fixed, then $K\in\mathscr{K}\left(\mu\right)$ if and
only if 
\begin{equation}
F_{\varphi}:=\int K\left(\cdot,y\right)\varphi\left(y\right)\mu\left(dy\right)\in\mathscr{H}_{K},\;\forall\varphi\in L^{2}\left(\mu\right).\label{eq:a7}
\end{equation}
Further, (\ref{eq:a7}) is equivalent to the following: $\forall N\in\mathbb{N}$,
$\forall\left(\alpha_{i}\right)_{i=1}^{N}\subset\mathbb{C}$, $\forall\left(x_{i}\right)_{i=1}^{N}\subset X$,
$\exists C_{\varphi}<\infty$ with 
\begin{equation}
\left|\sum\alpha_{i}\int K\left(x_{i},y\right)\varphi\left(y\right)\mu\left(dy\right)\right|^{2}\leq C_{\varphi}\sum_{i}\sum_{j}\overline{\alpha_{i}}\alpha_{j}K\left(x_{i},x_{j}\right).\label{eq:a8}
\end{equation}
\end{lem}

\begin{proof}
Assume $F_{\varphi}\in\mathscr{H}_{K}$, for all $\varphi\in L^{2}\left(\mu\right)$.
Then 
\begin{align*}
\text{LHS}_{\left(\ref{eq:a8}\right)} & =\left|\left\langle \sum\alpha_{i}K_{x_{i}},F_{\varphi}\right\rangle \right|^{2}\\
 & \leq\left\Vert F_{\varphi}\right\Vert _{\mathscr{H}_{K}}^{2}\left\Vert \sum\alpha_{i}K_{x_{i}}\right\Vert _{\mathscr{H}_{K}}^{2}=\text{RHS}_{\left(\ref{eq:a8}\right)}
\end{align*}
with $C_{\varphi}=\left\Vert F_{\varphi}\right\Vert _{\mathscr{H}_{K}}^{2}$. 

Conversely, if (\ref{eq:a8}) holds, then 
\[
\sum\alpha_{k}K_{x_{i}}\longmapsto\sum\overline{\alpha_{i}}\int K\left(y,x_{i}\right)\overline{\varphi\left(y\right)}\mu\left(dy\right)
\]
extends to a unique bounded linear functional $l_{\varphi}$ on $\mathscr{H}_{K}$,
and so by Riesz, 
\[
l_{\varphi}\left(f\right)=\left\langle \xi,f\right\rangle _{\mathscr{H}_{K}},\quad\forall f\in\mathscr{H}_{K}.
\]
for some $\xi\in\mathscr{H}_{K}$. Setting $f=K_{x}$, then
\[
\overline{\xi\left(x\right)}=l_{\varphi}\left(K_{x}\right)=\int K\left(y,x\right)\overline{\varphi\left(y\right)}\mu\left(dy\right)=\overline{F_{\varphi}\left(x\right)},\quad\forall x\in X.
\]
That is, $F_{\varphi}=\xi\in\mathscr{H}_{K}$. 
\end{proof}
\begin{lem}
If $K\left(\cdot,\cdot\right)$ is an integral operator acting on
$L^{2}\left(X,\mathscr{B}_{X},\mu\right)$, where $\mu$ is $\sigma$-finite,
then it is positive definite if and only if 
\begin{equation}
\int_{X}\int_{X}\overline{\varphi\left(x\right)}K\left(x,y\right)\varphi\left(y\right)\geq0,\quad\forall\varphi\in L^{2}\left(\mu\right).\label{eq:a9}
\end{equation}
\end{lem}

\begin{rem}
It is easy to check (\ref{eq:a9}) for $\varphi=\sum_{i}\alpha_{i}\chi_{B_{i}}$,
$\left\{ x_{i}\right\} _{i=1}^{N}\subset X$, $\left\{ B_{i}\right\} _{i=1}^{N}\subset\mathscr{B}_{X}$.
Then (\ref{eq:a9}) is equivalent to 
\[
\sum\sum\overline{\alpha_{i}}\alpha_{j}\mu\left(B_{i}\right)\mu\left(B_{j}\right)K\left(x_{i},x_{j}\right)\geq0.
\]
\end{rem}

\begin{lem}
If $\left\{ f_{i}\right\} $ is an ONB (or a frame) in $\mathscr{H}_{K}$,
then 
\[
K\left(x,y\right)=\sum_{i}f_{i}\left(x\right)\overline{f_{i}\left(y\right)},\quad\forall\left(x,y\right)\in X\times X;
\]
and 
\begin{align*}
\left(T_{\mu,K}T_{\mu,K}^{*}\varphi\right)\left(x\right) & =\sum_{i}\left(\int\overline{f_{i}\left(y\right)}\varphi\left(y\right)\mu\left(dy\right)\right)f_{i}\left(x\right)\\
 & =\sum_{i}\left\langle \varphi,f_{i}\right\rangle _{L^{2}\left(\mu\right)}f_{i}\left(x\right).
\end{align*}
Moreover, 
\begin{align*}
f^{\mu,K} & :=\mathop{\text{argmin}}\left\{ \left\Vert \varphi-T_{\mu,K}f\right\Vert _{L^{2}\left(\mu\right)}^{2}+\alpha\left\Vert f\right\Vert _{\mathscr{H}_{K}}^{2}\right\} \\
 & =\sum\left\langle f_{i},f^{\mu,K}\right\rangle _{\mathscr{H}_{K}}f_{i}.
\end{align*}
\end{lem}

Given all admissible pairs $\left(K,\mu\right)$, let $T_{\mu,K}$
and $T_{\mu,K}^{*}$ be as above. There are two selfadjoint operators
(possibly unbounded): 
\begin{align}
L^{2}\left(\mu\right) & \xrightarrow{\quad T_{\mu,K}T_{\mu,K}^{*}\quad}L^{2}\left(\mu\right)\\
\mathscr{H}_{K} & \xrightarrow{\quad T_{\mu,K}^{*}T_{\mu,K}\quad}\mathscr{H}_{K}
\end{align}
In particular, 
\begin{equation}
spec\left(T_{\mu,K}T_{\mu,K}^{*}\right)\cup\left\{ 0\right\} =spec\left(T_{\mu,K}^{*}T_{\mu,K}\right)\cup\left\{ 0\right\} ,
\end{equation}
which holds in general. 
\begin{lem}
\label{lem:a9}Let $\mu=\delta_{x_{0}}$, then 
\begin{align*}
\left(T_{\delta_{x_{0}}}T_{\delta_{x_{0}}}^{*}f\right)\left(\cdot\right) & =K\left(\cdot,x_{0}\right)f\left(x_{0}\right),\\
\left(T_{\delta_{x_{0}}}^{*}T_{\delta_{x_{0}}}K\left(\cdot,x\right)\right)\left(z\right) & =\overline{K\left(x_{0},z\right)}K\left(x_{0},x\right),\quad\forall\left(x,z\right)\in X\times X.
\end{align*}
\end{lem}

\begin{proof}
One checks that 
\[
\left(T_{\delta_{x_{0}}}^{*}f\right)\left(\cdot\right)=\int K\left(\cdot,y\right)f\left(y\right)\mu\left(dy\right)=K\left(\cdot,x_{0}\right)f\left(x_{0}\right),
\]
and so 
\[
T_{\delta_{x_{0}}}^{*}T_{\delta_{x_{0}}}K\left(\cdot,x\right)=K\left(\cdot,x_{0}\right)K\left(x_{0},x\right).
\]
\end{proof}

\section{\label{sec:OS}Optimal feature selections}

In the remaining of the paper, we formulate three versions of the
general optimization problem as presented in outline in \secref{OW},
i.e., the problem optimization over suitable choices of convex sets
of kernels $K$. In brief outline, the three variants are as follows:
(i) In \secref{OS}, the optimalization entails just the $\mathscr{H}_{K}$-norm\textsuperscript{2}
applied to the optimal $f^{\varphi,K}$ from \lemref{B2}. (ii) In
\secref{OA}, a different measure of \textquotedblleft optimal\textquotedblright{}
is used, with solution formula as in \corref{d2}. Finally, (iii)
in \secref{CME}, our optimization is obtained, and it makes use of
the CME approach.

Hence \secref{OS} presents some cases of non-existence of optimizers.
This in turn serves to motivate our affirmative optimization results
in Sections \ref{sec:OA} \& \ref{sec:CME}, especially \thmref{d4},
\corref{d5}, and \thmref{f7}.

Our assumptions below remain as mentioned above. In particular, a
fixed function $\varphi$ is specified (representing \textquotedblleft training
data.\textquotedblright ). Also given is a positive sigma-finite measure
$\mu$. We assume $\varphi\in L^{2}\left(\mu\right)$. Our analysis
of feature selection is based on both regression starting with a p.d.
kernel $K$, as well as a variation for choices of p.d. kernels $K$.
Each $K$ yields a selection of admissible features. But a \textquotedblleft good\textquotedblright{}
choice of $K$ yields corresponding optimal feature functions $f^{\varphi,K}$,
thus representing more distinct features, reflected in feature functions
$f^{\varphi,K}$ with large $\mathscr{H}_{K}$-norm\textsuperscript{2},
i.e., large variance. Optimal choices of $K$ typically represent
more successful discrimination by features resulting from input of
a particular training data, the function $\varphi$. More precisely,
the $\mathscr{H}_{K}$-norm\textsuperscript{2} refers to the features
entailed by a choice of $K$. By contrast, the training data represented
by $\varphi$ is fixed.
\begin{lem}
\label{lem:c1}With $\mu,K$ fixed, $K\in\mathscr{K}\left(\mu\right)$.
Let $f^{\varphi,K}$ be as specified in (\ref{eq:a2}). Then 
\begin{align}
\left\Vert f^{\varphi,K}\right\Vert _{\mathscr{H}_{K}}^{2} & =\left\langle \varphi,T_{\mu,K}T_{\mu,K}^{*}\left(\alpha+T_{\mu,K}T_{\mu,K}^{*}\right)^{-2}\varphi\right\rangle _{L^{2}\left(\mu\right)}\\
 & =\left\Vert \left(T_{\mu,K}T_{\mu,K}^{*}\right)^{1/2}\left(\alpha+T_{\mu,K}T_{\mu,K}^{*}\right)^{-1}\varphi\right\Vert _{L^{2}\left(\mu\right)}^{2}\label{eq:C2}\\
 & =\int\frac{x}{\left(\alpha+x\right)^{2}}\left\Vert Q^{\mu,K}\left(dx\right)\varphi\right\Vert _{L^{2}\left(\mu\right)}^{2},
\end{align}
where $Q^{\mu,K}\left(\cdot\right)$ is the spectral measure of the
operator $T_{\mu,K}T_{\mu,K}^{*}$, i.e., 
\begin{equation}
T_{\mu,K}T_{\mu,K}^{*}=\int_{0}^{\infty}x\,Q^{\mu,K}\left(dx\right).
\end{equation}
\end{lem}

\begin{proof}
Let $T:=T_{\mu,K}$. Note that $TT^{*}\left(\alpha+TT^{*}\right)^{-2}:L^{2}\left(\mu\right)\rightarrow L^{2}\left(\mu\right)$
is a bounded operator. We have 
\begin{eqnarray*}
 &  & \left\langle T^{*}\left(\alpha+TT^{*}\right)^{-1}\varphi,T^{*}\left(\alpha+TT^{*}\right)^{-1}\varphi\right\rangle _{\mathscr{H}_{K}}\\
 & = & \left\langle \left(\alpha+TT^{*}\right)^{-1}\varphi,TT^{*}\left(\alpha+TT^{*}\right)^{-1}\varphi\right\rangle _{L^{2}\left(\mu\right)}\\
 & = & \left\langle \varphi,TT^{*}\left(\alpha+TT^{*}\right)^{-2}\varphi\right\rangle _{L^{2}\left(\mu\right)}.
\end{eqnarray*}
\end{proof}
\begin{rem}
Note that, if $\left\Vert T_{\mu,K}T_{\mu,K}^{*}\right\Vert <\alpha$,
then the function $K\longmapsto\left\Vert f^{\varphi,K}\right\Vert _{\mathscr{H}_{K}}^{2}$is
monotone relative to the order of kernels: $K\ll K'\Longleftrightarrow\int\varphi K\varphi\,d\mu\leq\int\varphi K'\varphi\,d\mu$.
In that case, we need only optimize with respect to the spectral measure
of the kernel $K\in\mathscr{K}\left(\mu\right)$, with $\mu$ fixed. 
\end{rem}

\begin{example}
If $\mu=\delta_{x_{0}}$ as in \lemref{a9}, then 
\[
\left(T_{\delta_{x_{0}}}T_{\delta_{x_{0}}}^{*}\psi\right)\left(\cdot\right)=K\left(\cdot,x_{0}\right)\psi\left(x_{0}\right)\in L^{2}\left(X,\delta_{x_{0}}\right).
\]
And 
\[
\left\langle \varphi,T_{\delta_{x_{0}}}T_{\delta_{x_{0}}}^{*}\psi\right\rangle _{L^{2}\left(\delta_{x_{0}}\right)}=K\left(x_{0},x_{0}\right)\overline{\varphi\left(x_{0}\right)}\psi\left(x_{0}\right).
\]

Similarly, for pure atomic measures $\mu=\sum_{i}\alpha_{i}\delta_{x_{i}}$,
we have 
\[
\left\langle \varphi,T_{\mu}T_{\mu}^{*}\psi\right\rangle _{L^{2}\left(\mu\right)}=\sum_{i}\sum_{j}\overline{\alpha_{i}}\alpha_{j}K\left(x_{i},x_{j}\right)\overline{\varphi\left(x_{i}\right)}\psi\left(x_{j}\right)
\]
for all $\varphi,\psi\in L^{2}\left(\mu\right)$. 
\end{example}

Below we fix a positive measure $\mu$ as per \defref{b3}. When an
ONB is fixed in the corresponding $L^{2}(\mu)$ we then arrive at
a convex set $C_{\mu}$ of Mercer kernels $K$, see (\ref{eq:c5}):
$C_{\mu}$ is specified as in (\ref{eq:c5}) below. So, these p.d.
kernels $K$, and the corresponding RKHSs, are determined by the spectral
data (\ref{eq:c8}). As a consequence, we see that the optimal feature
vector may be found via a solution to this convex optimization problem
for $K$ in $C_{\mu}$. Further note that the spectral data used in
the case (of Mercer kernels) is a special case of the general structure
presented in \lemref{c1} above. Indeed, the reader can verify that
the optimization algorithm presented below for the case of Mercer
kernels generalizes to more general cases of convex sets of p.d. kernels
as per \lemref{c1} above.
\begin{thm}
\label{thm:c4}Fix $\mu$, and let $K\in\mathscr{K}\left(\mu\right)$.
Let $\left\{ e_{i}\right\} _{i\in\mathbb{N}}$ be an ONB in $L^{2}\left(\mu\right)$,
and consider the Mercer kernel
\begin{equation}
K\left(x,y\right)=\sum\lambda_{i}e_{i}\left(x\right)e_{i}\left(y\right)\label{eq:c5}
\end{equation}
with $\lambda_{i}>0$, and $\sum\lambda_{i}=1$.

In this case, 
\begin{equation}
\left\langle \varphi,T_{K}T_{K}^{*}\psi\right\rangle _{L^{2}\left(\mu\right)}=\sum\lambda_{i}\left\langle \varphi,e_{i}\right\rangle _{L^{2}\left(\mu\right)}\left\langle e_{i},\psi\right\rangle _{L^{2}\left(\mu\right)}.
\end{equation}
Let $f^{K}$ be the optimal solution as in (\ref{eq:a2}). Then, 
\begin{equation}
\left\Vert f^{K}\right\Vert _{\mathscr{H}_{K}}^{2}=\sum\frac{\lambda_{i}}{\left(\alpha+\lambda_{i}\right)^{2}}\left|\left\langle \varphi,e_{i}\right\rangle \right|^{2}.\label{eq:c7}
\end{equation}

Moreover, consider the optimization problem: 
\begin{equation}
\left\{ \begin{alignedat}{1} & \max_{\left(\lambda_{i}\right)}\sum\frac{\lambda_{i}}{\left(\alpha+\lambda_{i}\right)^{2}}c_{i}\\
 & \sum\lambda_{i}=1,\quad\lambda_{i}\geq0\\
 & \sum c_{i}=\left\Vert \varphi\right\Vert _{L^{2}\left(\mu\right)}^{2},\quad c_{i}:=\left|\left\langle e_{i},\varphi\right\rangle \right|^{2}\geq0
\end{alignedat}
\right.\label{eq:c8}
\end{equation}
The solution $\left(\lambda_{i}^{max}\right)$ satisfies that 
\begin{equation}
\frac{\lambda_{i}^{max}}{\left(\alpha+\lambda_{i}^{max}\right)^{2}}=\xi c_{i}\label{eq:c9}
\end{equation}
for some constant $\xi\in\mathbb{R}_{+}$. 
\end{thm}

\begin{figure}
\includegraphics[width=0.45\textwidth]{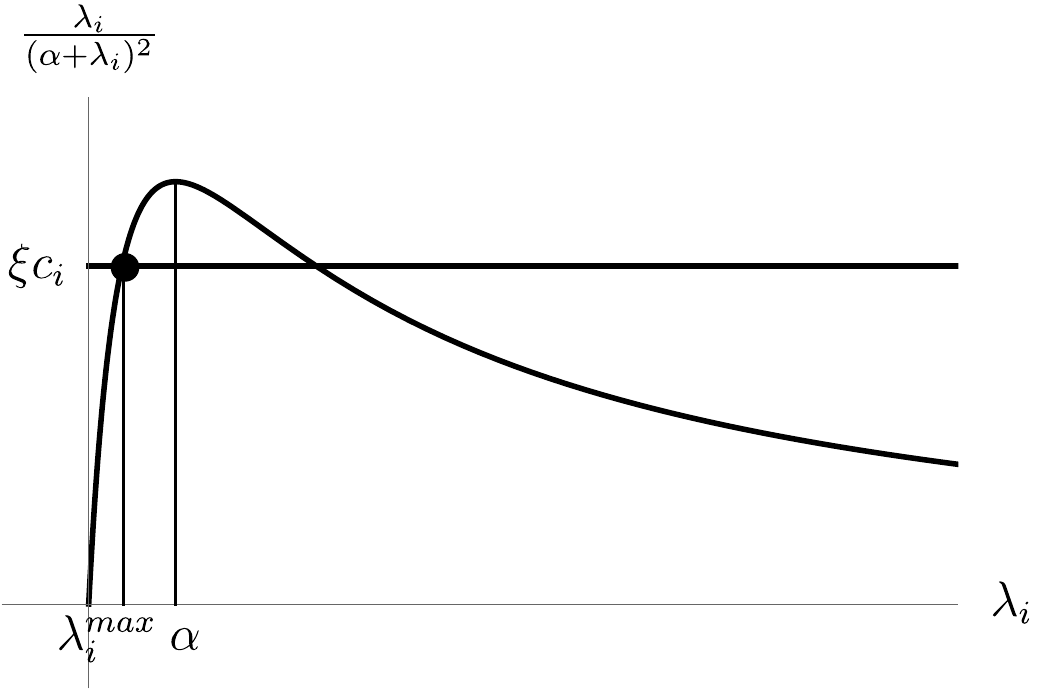}

\caption{\label{fig:sd}Spectral distribution in (\ref{eq:c9}).}
\end{figure}

\begin{proof}
The condition in (\ref{eq:c9}) follows from an application of the
Cauchy-Schwarz inequality. 

The fact that the solution $\left(\lambda_{i}^{max}\right)$ to (\ref{eq:c9})
in fact represents the solution to the optimization (\ref{eq:c8})
follows from the observation that the one term in the $l^{2}$ inner
product is fixed, so the max in (\ref{eq:c8}) is attained when quality
holds in the corresponding Cauchy-Schwarz Inequality. Further note
that, for every fixed value of the index $i$, (\ref{eq:c9}) is simply
a quadratic equation (see also \figref{sd}), and the optimal spectral
distribution $\left(\lambda_{i}^{max}\right)$ is explicit. The form
of the optimal p.d. kernel $K$ then follows by substitution of $\left(\lambda_{i}^{max}\right)$
into (\ref{eq:c5}).
\end{proof}
\begin{cor}
Consider the finite-dimensional case, i.e., $\mu$ is atomic, where
$K=\sum_{i=1}^{N}\lambda_{i}e_{i}\left(x\right)e_{i}\left(y\right)$,
with $\left\{ e_{i}\right\} _{i=1}^{N}$ an ONB in $L^{2}\left(\mu\right)$.
Then the optimization problem 
\begin{equation}
\left\{ \begin{alignedat}{1} & \max_{\left(\lambda_{i}\right)}\sum_{i=1}^{N}\frac{\lambda_{i}}{\left(\alpha+\lambda_{i}\right)^{2}}c_{i}\\
 & \sum_{i=1}^{N}\lambda_{i}=1,\quad\lambda_{i}\geq0\\
 & \sum_{i=1}^{N}c_{i}=\left\Vert \varphi\right\Vert _{L^{2}\left(\mu\right)}^{2},\quad c_{i}:=\left|\left\langle e_{i},\varphi\right\rangle \right|^{2}\geq0
\end{alignedat}
\right.
\end{equation}
has solution $\left(\lambda_{i}^{max}\right)$ determined by 
\begin{equation}
\left(\alpha-\lambda_{i}^{max}\right)c_{i}=A_{N}\left(\alpha+\lambda_{i}^{max}\right)^{3}.
\end{equation}
See \figref{root}.

Moreover, we have 
\begin{equation}
\left\Vert f_{N}^{K}\right\Vert _{\mathscr{H}_{K}}^{2}=\sum_{i=1}^{N}\frac{\lambda_{i}^{max}}{\left(\alpha+\lambda_{i}^{max}\right)^{2}}c_{i}=\sum_{i=1}^{N}\frac{A_{N}^{2/3}\lambda_{i}^{max}c_{i}^{1/3}}{\left(\alpha-\lambda_{i}^{max}\right)^{2/3}}.\label{eq:c12}
\end{equation}
\end{cor}

\begin{proof}
Let $L$ be the Lagrangian, where
\[
L=\sum_{i=1}^{N}\frac{\lambda_{i}}{\left(\alpha+\lambda_{i}\right)^{2}}c_{i}-A_{N}\left(\sum_{i=1}^{N}\lambda_{i}-1\right).
\]
Then, 
\begin{gather*}
\frac{\partial L}{\partial\lambda_{i}}=\frac{\alpha^{2}-\lambda_{i}^{2}}{\left(\alpha+\lambda_{i}\right)^{4}}c_{i}-A_{N}=0\\
\Updownarrow\\
\left(\alpha-\lambda_{i}\right)c_{i}=A_{N}\left(\alpha+\lambda_{i}\right)^{3}.
\end{gather*}
This yields 
\[
\left(\alpha+\lambda_{i}\right)^{2}=\frac{\left(\alpha-\lambda_{i}\right)^{2/3}c_{i}^{2/3}}{A_{N}^{2/3}}
\]
so that 
\[
\sum_{i=1}^{N}\frac{\lambda_{i}^{max}}{\left(\alpha+\lambda_{i}^{max}\right)^{2}}c_{i}=\sum_{i=1}^{N}\frac{\lambda_{i}^{max}}{\left(\alpha-\lambda_{i}^{max}\right)^{2/3}c_{i}^{2/3}}A_{N}^{2/3}c_{i}=\sum_{i=1}^{N}\frac{\lambda_{i}^{max}}{\left(\alpha-\lambda_{i}^{max}\right)^{2/3}}A_{N}^{2/3}c_{i}^{1/3}
\]
which is (\ref{eq:c12}).
\end{proof}
\begin{figure}
\includegraphics[width=0.45\columnwidth]{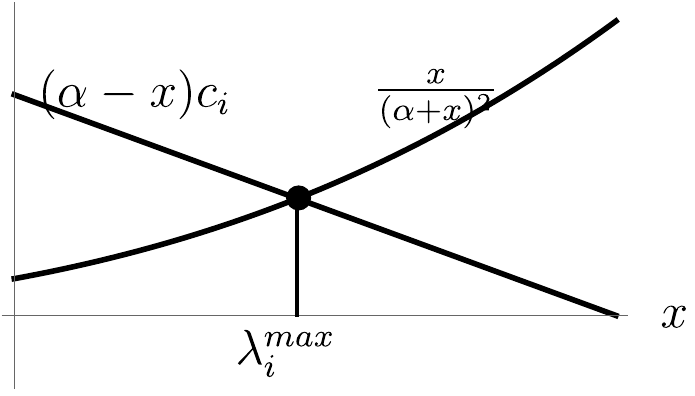}

\caption{\label{fig:root} The solution $\left(\lambda_{i}^{max}\right)$ determined
by the intersection of two curves.}

\end{figure}

\section{\label{sec:OA}Optimization in an ambient Hilbert space}

The setting below is as in the previous sections: input is specified
by two parts, first a fixed measure $\mu$, and secondly, an input
of functions $\varphi$ from $L^{2}\left(\mu\right)$, with the variety
of functions representing training data in the model. We then examine
optimal choices for p.d. kernels $K$ with view to optimization of
$K$-features for the corresponding kernel-learning, see \secref{OW}
above. The choices of optimal kernels $K$ are made precise in Theorems
\ref{thm:c4}, and in \ref{thm:d4} below. In both cases, an ONB in
$L^{2}(\mu)$ is chosen, and we study the corresponding convex sets
of Mercer kernels $K$ as specified in (\ref{eq:c5}). Here, then
each $K$ is determined by a spectral distribution $\left\{ \lambda_{i}\right\} $.
The corresponding optimization quantity is from (\ref{eq:a2}), and
it has a $\mathscr{H}_{K}$ penalty term weighted with an assigned
parameter $\alpha$, see (\ref{eq:a1}). We then arrive at an optimal
feature vector $f^{K}$ for every $K$, and we consider its $K$-variance,
measured with the use of the $\mathscr{H}_{K}$ norm-squared. Two
such variance measures are considered, (\ref{eq:c7}) and (\ref{eq:d14}).
For the first one, we note a singularity blowup for values of $\lambda_{i}$
close to $\alpha$. In the second case, the dependence on $K$ takes
a different form; we show that then the $K$-variance (see (\ref{eq:d14}))
is monotone, in a sense made precise in \thmref{d4}, and \corref{d5}
(spectral \emph{a priori} error-bounds).

Finally, in \secref{CME}, we present a solution to the feature optimization
with the use of a conditional-mean embedding, CME. The latter refers
to (i) a choice of probability space, (ii) a family of p.d. kernels
$L$, and corresponding conditional mean embeddings into the RKHSs
$\mathscr{H}_{L}$.

Below we first recall some basic facts from operator theory. Let $T:\mathscr{H}_{1}\rightarrow\mathscr{H}_{2}$
be a closed, densely defined linear operator between Hilbert spaces.
On $\mathscr{H}_{1}\times\mathscr{H}_{2}$, define the inner product
\begin{equation}
\left\langle \left(u_{1},v_{1}\right),\left(u_{2},v_{2}\right)\right\rangle _{\mathscr{H}_{1}\times\mathscr{H}_{2}}:=\alpha\left\langle u_{1},v_{1}\right\rangle _{\mathscr{H}_{1}}+\left\langle u_{2},v_{2}\right\rangle _{\mathscr{H}_{2}}\label{eq:d1}
\end{equation}
where $\alpha$ is a positive constant.

Define $W:\mathscr{H}_{1}\rightarrow\mathscr{H}_{1}\times\mathscr{H}_{2}$
by 
\[
W\left(u\right)=\left(u,Tu\right),\quad u\in dom\left(T\right),
\]
Then the projection from $\mathscr{H}_{1}\times\mathscr{H}_{2}$ onto
$ran\left(W\right)$ is 
\begin{equation}
\begin{bmatrix}\alpha\left(\alpha I_{1}+T^{*}T\right)^{-1} & T^{*}\left(\alpha I_{2}+TT^{*}\right)^{-1}\\
\alpha T\left(\alpha I_{1}+T^{*}T\right)^{-1} & TT^{*}\left(\alpha I_{2}+TT^{*}\right)^{-1}
\end{bmatrix}.\label{eq:d2}
\end{equation}

\begin{rem}
Recall that in our general setup for regression optimization, we have
arranged that training data may be represented via an operator $T$
in Hilbert space. Note that, if $\alpha=1$, then the block matrix
in (\ref{eq:d2}) represents the projection from the direct sum $\mathscr{H}_{1}\oplus\mathscr{H}_{2}$
onto the graph of the operator $T$, see e.g., \cite[Corollary 1.55]{MR4274591}.
\end{rem}

\begin{cor}
\label{cor:d2}Let $T:\mathscr{H}_{1}\rightarrow\mathscr{H}_{2}$
be as above. Then, for all $v\in\mathscr{H}_{2}$, we have 
\begin{align}
u^{*} & =\mathbb{\text{argmin}}\left\{ \alpha\left\Vert u\right\Vert _{\mathscr{H}_{1}}^{2}+\left\Vert Tu-v\right\Vert _{\mathscr{H}_{2}}^{2}:u\in\mathscr{H}_{1}\right\} \nonumber \\
 & =\mathop{\text{argmin}}\left\{ \left\Vert Wu-\left(0,v\right)\right\Vert _{\mathscr{H}_{1}\times\mathscr{H}_{2}}:u\in\mathscr{H}_{1}\right\} \nonumber \\
 & =T^{*}\left(\alpha I_{2}+TT^{*}\right)^{-1}v.\label{eq:d5}
\end{align}
\end{cor}

\begin{proof}
Note that
\[
\alpha\left\Vert u\right\Vert _{\mathscr{H}_{1}}^{2}+\left\Vert Tu-v\right\Vert _{\mathscr{H}_{2}}^{2}=\left\Vert Wu-\left(0,v\right)\right\Vert _{\mathscr{H}_{1}\times\mathscr{H}_{2}}
\]
and the projection of $\left(0,v\right)$ onto $ran\left(W\right)$
is 
\[
\left(T^{*}\left(\alpha I_{2}+TT^{*}\right)^{-1}v,TT^{*}\left(\alpha I_{2}+TT^{*}\right)^{-1}\right)
\]
which is equal to $Wu^{*}=\left(u^{*},Tu^{*}\right)$, for a unique
$u^{*}$ in $\mathscr{H}_{1}$. This gives (\ref{eq:d5}).
\end{proof}
Now, return to optimal feature selections. Fix $\mu$, and consider
kernels $K\in\mathscr{K}\left(\mu\right)$, see \defref{K}. Let $W_{K}:\mathscr{H}_{K}\rightarrow\mathscr{H}_{K}\times L^{2}\left(\mu\right)$,
by 
\[
W_{K}h=\left(h,T_{K,\mu}h\right).
\]
The inner product on $\mathscr{H}_{K}\times L^{2}\left(\mu\right)$
is as in (\ref{eq:d1}), with parameter $\alpha>0$. 

Fix $\varphi\in L^{2}\left(\mu\right)$, then we get a unique $f^{\varphi,K}\in\mathscr{H}_{K}$,
such that $W_{K}f^{\varphi,K}$ is the projection of $\left(0,\varphi\right)\in\mathscr{H}_{K}\times L^{2}\left(\mu\right)$
onto $ran\left(W_{K}\right)$. That is, 
\begin{equation}
f^{\varphi,K}=T_{K,\mu}^{*}\left(\alpha I_{L^{2}}+T_{K,\mu}T_{K,\mu}^{*}\right)^{-1}\varphi\label{eq:d9}
\end{equation}
by \corref{d2}.

This motivates the optimization problem: 
\begin{gather*}
\max_{K\in\mathscr{K}\left(\mu\right)}\left\{ \left\Vert W_{K}f^{\varphi,K}\right\Vert _{\mathscr{H}_{K}\times L^{2}\left(\mu\right)}^{2}\right\} \\
\Updownarrow\\
\max_{K\in\mathscr{K}\left(\mu\right)}\left\{ \alpha\left\Vert f^{\varphi,K}\right\Vert _{\mathscr{H}_{K}}^{2}+\left\Vert T_{K,\mu}f^{\varphi,K}\right\Vert _{L^{2}\left(\mu\right)}^{2}\right\} 
\end{gather*}

Below is a modification of \lemref{c1}.
\begin{lem}
\label{lem:d3}With $\mu,K$ fixed, $K\in\mathscr{K}\left(\mu\right)$.
Let $f^{\varphi,K}$ be as specified above. Then 
\begin{equation}
\left\Vert W_{K}f^{\varphi,K}\right\Vert _{\mathscr{H}_{K}\times L^{2}\left(\mu\right)}^{2}=\int\frac{x}{\alpha+x}\left\Vert Q^{\mu,K}\left(dx\right)\varphi\right\Vert _{L^{2}\left(\mu\right)}^{2},\label{eq:d11}
\end{equation}
where $Q^{K,\mu}\left(dx\right)$ is the spectral projection of the
operator $T_{K,\mu}T_{K,\mu}^{*}$. 

Especially, 
\begin{equation}
\left\Vert W_{K}f^{\varphi,K}\right\Vert _{\mathscr{H}_{K}\times L^{2}\left(\mu\right)}^{2}\leq\left\Vert \varphi\right\Vert _{L^{2}\left(\mu\right)}^{2}.\label{eq:d12}
\end{equation}
\end{lem}

\begin{proof}
Let $T:=T_{K,\mu}$, then 
\begin{align*}
\left\Vert W_{K}f^{\varphi,K}\right\Vert _{\mathscr{H}_{K}\times L^{2}\left(\mu\right)}^{2} & =\alpha\left\Vert f^{\varphi,K}\right\Vert _{\mathscr{H}_{K}}^{2}+\left\Vert Tf^{\varphi,K}\right\Vert _{L^{2}\left(\mu\right)}^{2}\\
 & =\alpha\left\Vert T^{*}\left(\alpha+TT^{*}\right)^{-1}\varphi\right\Vert _{\mathscr{H}_{K}}^{2}+\left\Vert \left(\alpha+TT^{*}\right)^{-1}TT^{*}\varphi\right\Vert _{L^{2}\left(\mu\right)}^{2}\\
 & =\int\left(\frac{\alpha x}{\left(\alpha+x\right)^{2}}+\frac{x^{2}}{\left(\alpha+x\right)^{2}}\right)\left\Vert Q^{\mu,K}\left(dx\right)\varphi\right\Vert _{L^{2}\left(\mu\right)}^{2}\\
 & =\int\frac{x}{\alpha+x}\left\Vert Q^{\mu,K}\left(dx\right)\varphi\right\Vert _{L^{2}\left(\mu\right)}^{2}.
\end{align*}
Note, (\ref{eq:d12}) holds, since $W_{K}f^{\varphi,K}$ is the projection
of $\left(0,\varphi\right)\in\mathscr{H}_{K}\times L^{2}\left(\mu\right)$
onto the range of $W$. 
\end{proof}
Now we state an analog of \thmref{c4}.
\begin{thm}
\label{thm:d4}Fix $\mu$, and let $K\in\mathscr{K}\left(\mu\right)$.
Let $\left\{ e_{i}\right\} _{i\in\mathbb{N}}$ be an ONB in $L^{2}\left(\mu\right)$,
and consider the p.d. kernel
\begin{equation}
K\left(x,y\right)=\sum\lambda_{i}e_{i}\left(x\right)e_{i}\left(y\right)\label{eq:d13}
\end{equation}
with $\lambda_{i}>0$. 

Let $f^{\varphi,K}$ be the optimal solution as in (\ref{eq:d9}).
Then, 
\begin{equation}
\left\Vert W_{K}f^{\varphi,K}\right\Vert _{\mathscr{H}_{K}\times L^{2}\left(\mu\right)}^{2}=\sum\frac{\lambda_{i}}{\alpha+\lambda_{i}}\left|\left\langle \varphi,e_{i}\right\rangle \right|^{2}.\label{eq:d14}
\end{equation}
\end{thm}

\begin{proof}
See the proof of \thmref{c4}.
\end{proof}
\begin{cor}
\label{cor:d5}Let $W_{K}$ and $f^{\varphi,K}$ be as in \thmref{d4},
and assume $K$ is bounded. Let $\lambda_{-}=\inf\left\{ \lambda_{i}\right\} $,
$\lambda_{+}=\sup\left\{ \lambda_{i}\right\} $. 
\begin{enumerate}
\item The following hold: 
\begin{equation}
\frac{\lambda_{-}}{\alpha+\lambda_{-}}\left\Vert \varphi\right\Vert _{L^{2}\left(\mu\right)}^{2}\leq\left\Vert W_{K}f^{\varphi,K}\right\Vert _{\mathscr{H}_{K}\times L^{2}\left(\mu\right)}^{2}\leq\frac{\lambda_{+}}{\alpha+\lambda_{+}}\left\Vert \varphi\right\Vert _{L^{2}\left(\mu\right)}^{2}.\label{eq:d15}
\end{equation}
\item Equivalently, the approximation error satisfies 
\begin{equation}
\frac{1}{\alpha+\lambda_{+}}\left\Vert \varphi\right\Vert _{L^{2}\left(\mu\right)}^{2}\leq err\leq\frac{1}{\alpha+\lambda_{-}}\left\Vert \varphi\right\Vert _{L^{2}\left(\mu\right)}^{2}.\label{eq:d16}
\end{equation}
\item By increasing $\lambda_{-}$, $W_{K}f^{\varphi,K}$ approximates $\left(0,\varphi\right)$
in $\mathscr{H}_{K}\times L^{2}\left(\mu\right)$ arbitrarily well. 
\end{enumerate}
\end{cor}

\begin{proof}
Notice that the function $f\left(x\right)=\frac{x}{\alpha+x}$ in
(\ref{eq:d11}) is strictly increasing in $(0,\infty)$, so that (\ref{eq:d15})
follows from (\ref{eq:d14}). The other assertions are immediate.
\end{proof}
\begin{rem}
The difference between the two feature selection methods in Sections
\ref{sec:OS} and \ref{sec:OA} is as follows. 

Fix a measure $\mu$ , and consider $K\in\mathscr{K}\left(\mu\right)$,
i.e., all admissible kernels. Let $\mathscr{H}_{K}$ be the associated
RKHS. In both cases, for a given $\varphi\in L^{2}\left(\mu\right)$,
the best feature vector in $\mathscr{H}_{K}$ is the same 
\[
f^{K,\varphi}=T_{K}^{*}\left(\alpha+T_{K}T_{K}^{*}\right)^{-1}\varphi.
\]
See (\ref{eq:a2}) and (\ref{eq:d9}). 

However, the criteria for optimization over kernels $K$ is different:
\begin{alignat}{2}
\text{Section \ref{sec:OS}} & \quad &  & \max_{K\in\mathscr{K}\left(\mu\right)}\left\{ \left\Vert f^{K,\varphi}\right\Vert _{\mathscr{H}_{K}}^{2}\right\} \label{eq:D11}\\
\text{Section \ref{sec:OA}} &  &  & \max_{K\in\mathscr{K}\left(\mu\right)}\left\{ \alpha\left\Vert f^{K,\varphi}\right\Vert _{\mathscr{H}_{K}}^{2}+\left\Vert T_{K}f^{K,\varphi}\right\Vert _{L^{2}}^{2}\right\} \label{eq:D12}
\end{alignat}
As discussed at the beginning of (\ref{sec:OA}), the vector 
\[
\left(f^{K,\varphi},T_{K}f^{K,\varphi}\right)
\]
is the projection of $\left(0,\varphi\right)$ in $\mathscr{H}_{K}\times L^{2}$
onto the graph of $T_{K}$. Thus, (\ref{eq:D12}) is the norm squared
of the projected vector and the corresponding optimization makes use
of Hilbert space geometry. 
\end{rem}

\section{\label{sec:CME}Applications to CME}

A key feature in what is called \emph{conditional mean embedding}
(CME) concern an analysis of systems of random variable, and conditional
distributions, which take values in suitable choices of reproducing
kernel Hilbert space, typically infinite-dimensional RKHSs. Hence
conditional expectations, and relative transition operators, will
entail choices of p.d. kernels, typical one for each random variable
under consideration. The implementation of kernel embedding of distributions
(also called the \emph{kernel mean} or mean map) yields nonparametric
outcomes in which a probability distribution is represented as an
element of a reproducing kernel Hilbert space (RKHS). In diverse applications,
the use of CMEs has served as useful tools in for example, problems
of sequentially optimizing conditional expectations for objective
functions. In such settings, typically both the conditional distribution
and the objective function, while fixed, are assumed to be unknown.

The assumption is that input is fixed in the form of a pair $\varphi$
(generalized training data) and $\mu$ as described. The interpretation
for feature selection is that variation of choices of p.d. kernels
$K$ amounts to a more versatile feature selection. A possible condition
on a \textquotedblleft good\textquotedblright{} kernel $K$ is that
it will yield optimal selection of feature function $f^{K}$, i.e.,
representing an output of more distinct features. Often a feature
functions $f^{K}$ with large $\mathscr{H}_{K}$-norm\textsuperscript{2}
comes from a choice of $K$ that yields a more successful discrimination
by features that reflects input of training data via $\varphi$. More
precisely, the $\mathscr{H}_{K}$-norm\textsuperscript{2} refers
to the features selected with optimal choices of $K$.

The setting for CME is as follows: 

Let $X,Y$ be random variables on a probability space $\left(\Omega,\mathscr{C},\mathbb{P}\right)$,
taking values in sets $A,B$, respectively, and has joint measure
\[
\mu\left(S_{1}\times S_{1}\right)=\mathbb{P}\left(X^{-1}\left(S_{1}\right)\cap Y^{-1}\left(S_{2}\right)\right)
\]
for all $S_{1}\times S_{2}\subset A\times B$.

Denote by $\mu_{X},\mu_{Y}$ the corresponding marginal measures,
and let $\mu_{Y\mid x}$ be the conditional measure defined as 
\[
\mu_{Y\mid x}\left(S\right)=\mathbb{P}\left(Y^{-1}\left(S\right)\mid X=x\right)
\]
for all $S\subset B$ and $x\in A$.

Assume further that $K,L$ are given p.d. kernels on $A$, $B$ with
RKHSs $\mathscr{H}_{K}$, $\mathscr{H}_{L}$, respectively. 
\begin{lem}
\label{lem:f1}For every $x\in A$, set 
\begin{equation}
\pi\left(x\right):=\mathbb{E}\left[L\left(\cdot,Y\right)\mid X=x\right]=\int L\left(\cdot,y\right)d\mu_{Y\mid x}\left(y\right)\label{eq:f1}
\end{equation}
Then, for all $f\in\mathscr{H}_{L}$, it holds that 
\begin{equation}
\left\langle \pi\left(x\right),f\right\rangle _{\mathscr{H}_{L}}=\int\left\langle L\left(\cdot,y\right),f\right\rangle d\mu_{Y\mid x}\left(y\right)=\mathbb{E}\left[f\left(Y\right)\mid X=x\right].
\end{equation}
The map $x\mapsto\pi\left(x\right)$ is called the kernel mean embedding
(KME) of the conditional expectation $\mathbb{E}\left[\;\cdot\mid X=x\right]$.
\end{lem}

\begin{rem}
Note that the integral on the RHS in formula (\ref{eq:f1}) is an
extension of (\ref{eq:a5}) from \lemref{b5} above. Moreover, the
proof of the lemma follows the ideas in \secref{OW}.
\end{rem}

\begin{lem}
As in \lemref{f1}, consider $\pi\left(x\right)$ for $x\in A$, as
per the definition (\ref{eq:f1}) in \lemref{f1}. Then
\begin{enumerate}
\item $\pi\left(x\right)\in\mathscr{H}_{L}$ if and only if 
\begin{equation}
\mathbb{E}\left[L\left(Y,Y\right)\mid X=x\right]:=\mu_{Y\mid x}L\mu_{Y\mid x}<\infty,\label{eq:f3}
\end{equation}
where
\[
\mu_{Y\mid x}L\mu_{Y\mid x}=\iint L\left(y_{1},y_{2}\right)d\mu_{Y\mid x}\left(y_{1}\right)d\mu_{Y\mid x}\left(y_{2}\right).
\]
\item $\pi\in L^{2}\left(A,\mu_{X}\right)\otimes\mathscr{H}_{L}$ if and
only if 
\[
\mathbb{E}\left[K\left(Y,Y\right)\right]=\int\left(\mu_{Y\mid x}L\mu_{Y\mid x}\right)d\mu_{X}\left(x\right)<\infty.
\]
In that case, setting $\widetilde{\pi}\left(f\right)\left(x\right):=\left\langle \pi\left(x\right),f\right\rangle _{\mathscr{H}_{L}}$,
then 
\[
\int\left\Vert \widetilde{\pi}\left(f\right)\right\Vert _{\mathscr{H}_{K}}^{2}d\mu_{X}\leq\mathbb{E}\left[K\left(Y,Y\right)\right]\left\Vert f\right\Vert _{\mathscr{H}_{L}}^{2}
\]
for all $f\in\mathscr{H}_{L}$.
\end{enumerate}
\end{lem}

\begin{proof}
Consider the filter of finite measurable partitions $\mathscr{P}\left(B\right)$
of the measurable space $\left(B,\mathscr{B}\right)$, i.e., $\left\{ A_{i}\right\} _{i=1}^{N}$
for some $N<\infty$, with $A_{i}\in\mathscr{B}$, $A_{i}\cap A_{j}=\emptyset$
if $i\neq j$, and $\cup_{i}A_{i}=B$, then 
\begin{equation}
\sum_{i=1}^{N}L\left(\cdot,y_{i}\right)\mu_{Y\mid x}\left(A_{i}\right)\in\mathscr{H}_{K}\label{eq:f4}
\end{equation}
with 
\begin{equation}
\left\Vert \sum\nolimits _{i=1}^{N}L\left(\cdot,y_{i}\right)\mu_{Y\mid x}\left(A_{i}\right)\right\Vert _{\mathscr{H}_{K}}^{2}=\sum\nolimits _{i}\sum\nolimits _{j}\mu_{Y\mid x}\left(A_{i}\right)L\left(y_{i},y_{j}\right)\mu_{Y\mid x}\left(A_{j}\right).\label{eq:f5}
\end{equation}
Since $L$ is assumed measurable, the right-hand side of (\ref{eq:f5})
has a limit, as we pass to the limit of the filter of all measurable
partitions $\mathscr{P}\left(B\right)$, see (\ref{eq:f4}), and the
limit is well defined and finite if and only if (\ref{eq:f3}) holds.
This follows from the following computation:
\[
\iint\mu_{Y\mid x}\left(dy_{1}\right)L\left(y_{1},y_{2}\right)\mu_{Y\mid x}\left(dy_{2}\right)=\mathbb{E}\left[L\left(Y,Y\right)\mid X=x\right].
\]
But since we have ``$=$'' in the identity (\ref{eq:f5}) for all
finite partitions, it follows that (\ref{eq:f3}) holds if and only
if the integral on the right-hand side in (\ref{eq:f1}) is convergent
with its values in $\mathscr{H}_{L}$. 

The second part of the lemma is immediate. 
\end{proof}
\begin{rem}
The setting of the lemma is a fixed a p.d. kernel $L$ and a measure
space $\left(B,\mathscr{B}\right)$. We have $L$ defined on $B\times B$
and assumed measurable w.r.t. the corresponding product sigma algebra.
The key idea behind the justification of the RKHS $\mathscr{H}_{L}$
valued integral $\pi(x)$ in (\ref{eq:f1}) is a rigorous justification
of a limit of an approximation by finite sums in $\mathscr{H}_{L}$,
and the limit with respect to the RKHS norm in $\mathscr{H}_{L}$.
This is doable as per our discussion, but the limit will be indexed
by a filter of partitions of the measure space $\left(B,\mathscr{B}\right)$.
And the limit is with respect to refinement within the filter of partitions,
where refinement defined by recursive subdivision, i.e., subdivisions
of one partition are creating a finer partition. Note that the reasoning
involves the same kind of limit which is used in the justification
of general Ito isometries, and Ito integrals for Gaussian processes. 
\end{rem}

\begin{question}
Assume $\pi\in L^{2}\left(A,\mu_{X}\right)\otimes\mathscr{H}_{L}$.
What is the best approximation to choice of CME $\mu$ from an $\mathscr{H}_{L}$-valued
RKHS? 
\end{question}

One option in the literature is to approximate $\pi$ from $\mathscr{H}_{K}\otimes\mathscr{H}_{L}$.
More generally, one may start from an $\mathscr{B}\left(\mathscr{H}_{L}\right)$-valued
p.d. kernel $S:A\times A\rightarrow\mathscr{B}\left(\mathscr{H}_{L}\right)$,
i.e., 

\begin{equation}
\sum_{i,j=1}^{N}\left\langle u_{i},S\left(x_{i},x_{j}\right)u_{j}\right\rangle _{\mathscr{H}_{L}}\geq0
\end{equation}
$\forall\left(x_{i}\right)_{i=1}^{N}\subset A$, $\forall\left(u_{i}\right)_{i=1}^{N}\subset\mathscr{H}_{L}$,
and $\forall N\in\mathbb{N}$. 

Let $\mathscr{H}_{S}$ be the Hilbert completion of the set $span\left\{ S\left(\cdot,x\right)u:x\in A,u\in\mathscr{H}_{L}\right\} $
with respect to the inner product 
\begin{equation}
\left\langle \sum S\left(\cdot,x_{i}\right)u_{i},\sum S\left(\cdot,x_{j}\right)v_{j}\right\rangle _{\mathscr{H}_{S}}:=\sum_{i,j}\left\langle u_{i},S\left(x_{i},x_{j}\right)v_{j}\right\rangle _{\mathscr{H}_{L}}.
\end{equation}
Then $\mathscr{H}_{S}$ is an RKHS with the following reproducing
property:

For all $F\in\mathscr{H}_{S}$, $x\in A$ and $u\in\mathscr{H}_{L}$,
we have
\begin{equation}
\left\langle u,F\left(x\right)\right\rangle _{\mathscr{H}_{L}}=\left\langle S\left(\cdot,x\right)u,F\right\rangle _{\mathscr{H}_{L}}.
\end{equation}

\begin{rem}
In the special case $\mathscr{H}_{S}=\mathscr{H}_{K}\otimes\mathscr{H}_{L}$,
we have $S\left(x,y\right)=K\left(x,y\right)I_{\mathscr{H}_{L}}$,
where $K$ is the scalar valued p.d. kernel of $\mathscr{H}_{K}$,
and $I_{\mathscr{H}_{L}}$ denotes the identity operator on $\mathscr{H}_{L}$. 
\end{rem}

\begin{thm}
\label{thm:f7}Assume $S$ is compatible with the marginal distribution
of $X$, then we have
\begin{align}
f^{\pi,S} & :=\mathop{\text{argmin}}\left\{ \left\Vert T_{S}f-\pi\right\Vert _{L^{2}\left(A\right)\otimes\mathscr{H}_{L}}^{2}+\alpha\left\Vert f\right\Vert _{\mathscr{H}_{S}}^{2}\right\} \\
 & =T_{S}^{*}\left(\alpha+T_{S}T_{S}^{*}\right)^{-1}\pi.\label{eq:d7}
\end{align}

Then, we may apply the methods from \secref{OS} to the problem: 
\begin{equation}
\max_{S}\left\Vert f^{\pi,S}\right\Vert _{\mathscr{H}_{S}}^{2}.
\end{equation}
\end{thm}

\begin{proof}
Illustrating the versatility of Hilbert space operators, the reader
will be able to fill in the argument for this formula (\ref{eq:d7}),
and its implications, following the general framework presented in
sections \ref{sec:OW} and \ref{sec:OS} above.
\end{proof}

\section{A new convex set of p.d. kernels}

Let $\mathcal{M}\left(\mathbb{R}\right)$ be the set of all Borel
measures on $\mathbb{R}$, and $\mathcal{M}_{1}\left(\mathbb{R}\right)$
be the subset of probability measures. For all $\rho\in\mathcal{M}\left(\mathbb{R}\right)$,
let 
\begin{equation}
\widehat{\rho}\left(\xi\right)=\int_{\mathbb{R}}e^{i\xi x}d\rho\left(x\right)
\end{equation}
denote the Fourier transform. 

Consider the following convex set of stationary kernels 
\begin{equation}
G_{1}=\left\{ \mathbb{R}\times\mathbb{R}\xrightarrow{\;K_{g}\;}\mathbb{C}:K_{g}\left(x,y\right)=g\left(x-y\right),\;g=\widehat{\mu},\:\mu\in\mathcal{M}_{1}\left(\mathbb{R}\right)\right\} .
\end{equation}

\begin{lem}
\label{lem:s1}Fix $K_{g}\in G_{1}$, and let $\mathscr{H}_{K_{g}}$
be the corresponding RKHS. Then, for all $\rho\in\mathcal{M}\left(\mathbb{R}\right)$,
\begin{gather}
g\ast d\rho:=\int_{\mathbb{R}}K_{g}\left(\cdot,y\right)d\rho\left(y\right)\in\mathscr{H}_{K_{g}}\\
\Updownarrow\nonumber \\
\int_{\mathbb{R}}\left|\widehat{\rho}\left(\xi\right)\right|^{2}d\mu\left(\xi\right)<\infty.
\end{gather}
\end{lem}

\begin{proof}
Assume $g\ast d\rho\in\mathscr{H}_{K_{g}}$, then 
\begin{align*}
\left\Vert g\ast d\rho\right\Vert _{\mathscr{H}_{K_{g}}}^{2} & =\iint\left\langle K_{g}\left(\cdot,y\right),K_{g}\left(\cdot,z\right)\right\rangle _{\mathscr{H}_{K_{g}}}d\rho\left(y\right)d\rho\left(z\right)\\
 & =\iint g\left(y-z\right)d\rho\left(y\right)d\rho\left(z\right)\\
 & =\int\left(\iint e^{i\xi\left(y-z\right)}d\rho\left(y\right)d\rho\left(z\right)\right)d\mu\left(\xi\right)\\
 & =\int\left|\widehat{\rho}\left(\xi\right)\right|^{2}d\mu\left(\xi\right)<\infty.
\end{align*}

Conversely, suppose $C:=\int\left|\widehat{\rho}\left(\xi\right)\right|^{2}d\mu\left(\xi\right)<\infty$.
Then, for all $\sum c_{k}K_{g}\left(\cdot,x_{k}\right)$, we have
\begin{eqnarray*}
\sum c_{k}K_{g}\left(\cdot,x_{k}\right) & \longmapsto & \left|\sum c_{k}\left(g\ast d\rho\right)\left(x_{k}\right)\right|^{2}\\
 & = & \left|\sum c_{k}\int\int e^{i\xi\left(x_{k}-y\right)}d\rho\left(y\right)d\mu\left(\xi\right)\right|^{2}\\
 & \leq & \int\left|\sum c_{k}e^{i\xi x_{k}}\right|\left|\widehat{\rho}\left(\xi\right)\right|d\mu\left(\xi\right)^{2}\\
 & \leq & \int\left|\sum c_{k}e^{i\xi x_{k}}\right|^{2}d\mu\left(\xi\right)\int\left|\widehat{\rho}\left(\xi\right)\right|^{2}d\mu\left(\xi\right)\\
 & = & C\cdot\sum_{k}\sum_{l}\overline{c_{k}}c_{l}K_{g}\left(x_{k},x_{l}\right).
\end{eqnarray*}
It follows that $g\ast d\rho\in\mathscr{H}_{K_{g}}$ by density and
Riesz's theorem. 
\end{proof}
Fix $K_{g}\in G_{1}$, and let $\mathscr{H}_{K_{g}}$ be the RKHS.
Let $d\lambda$ denote the Lebesgue measure on $\mathbb{R}$. Suppose
$\left\{ \varphi\in L^{2}\left(d\lambda\right):\widehat{\varphi}\in L^{2}\left(\mu\right)\right\} $
is dense in $L^{2}\left(d\lambda\right)$. Then, the operator 
\begin{equation}
T_{\lambda}:\mathscr{H}_{K_{g}}\rightarrow L^{2}\left(d\lambda\right),\quad T_{\lambda}\left(\sum_{i}c_{i}K_{g}\left(\cdot,x_{i}\right)\right)=\sum_{i}c_{i}K_{g}\left(\cdot,x_{i}\right)\label{eq:F5}
\end{equation}
is densely defined and closable, and its adjoint is given by 
\begin{equation}
T_{\lambda}^{*}:L^{2}\left(d\lambda\right)\rightarrow\mathscr{H}_{K_{g}},\quad T_{\lambda}^{*}\left(\varphi\right)=g\ast\varphi,\;\forall\varphi\in dom\left(T_{\lambda}^{*}\right)\label{eq:F6}
\end{equation}
where $dom\left(T_{\lambda}^{*}\right)=\left\{ \varphi\in L^{2}\left(d\lambda\right):\widehat{\varphi}\in L^{2}\left(\mu\right)\right\} $. 
\begin{cor}
For all $\varphi\in dom\left(T_{\lambda}^{*}\right)$, we have 
\begin{equation}
\left\Vert T_{\lambda}^{*}\varphi\right\Vert _{\mathscr{H}_{K_{g}}}^{2}=\int\left|\widehat{\varphi}\left(\xi\right)\right|^{2}d\mu\left(\xi\right).
\end{equation}
\end{cor}

\begin{proof}
This follows from \lemref{s1} by setting $d\rho=gd\lambda$, and
$\widehat{\varphi}$ is the $L^{2}$-Fourier transform of $\varphi$. 
\end{proof}
\begin{example}
\label{exa:g3}Consider the following two p.d. kernels on $\mathbb{R}$:
\begin{equation}
K_{1}\left(x,y\right)=e^{-\left|x-y\right|},\quad,K_{2}\left(x,y\right)=e^{-\frac{1}{2}\left(x-y\right)^{2}}.\label{eq:g8}
\end{equation}
Note that 
\begin{align*}
g_{1}\left(x\right) & =e^{-\left|x\right|}=\int e^{i\xi x}d\mu_{1}\left(\xi\right)\quad d\mu_{1}\left(\xi\right)=\frac{1}{\pi}\frac{1}{1+\xi^{2}}d\xi,\\
g_{2}\left(x\right) & =e^{-\frac{1}{2}x^{2}}=\int e^{i\xi x}d\mu_{2}\left(\xi\right)\quad d\mu_{2}\left(\xi\right)=\frac{1}{\sqrt{2\pi}}e^{-\frac{1}{2}\xi^{2}}d\xi.
\end{align*}
Moreover, for $K_{1}$, if $\varphi\in L^{2}\left(\mathbb{R}\right)$,
then 
\[
\left\Vert g_{1}\ast\varphi\right\Vert _{\mathscr{H}_{K_{1}}}^{2}=\int_{\mathbb{R}}\frac{\left|\widehat{\varphi}\left(\xi\right)\right|^{2}d\xi}{1+\xi^{2}}=\left\langle \varphi,\left(1-\left(d/dx\right)^{2}\right)^{-1}\varphi\right\rangle _{L^{2}\left(\mathbb{R}\right)}.
\]
In other words, the RKHS is the RKHS from the Green's function for
$1-\left(d/dx\right)^{2}$, or $1-\Delta$ in $\mathbb{R}^{n}$, $n>1$. 
\end{example}

Given $K_{g}\in G_{1}$, the convolution $\varphi\mapsto g\ast\varphi\in\mathscr{H}_{K_{g}}$
may be extended to measures or distributions.
\begin{lem}
Let $K_{g}\in G_{1}$, and $\mathscr{H}_{K_{g}}$ be the corresponding
RKHS. Then, 
\begin{equation}
g\left(x-\cdot\right)=g\ast\delta_{x}\in\mathscr{H}_{K_{g}}
\end{equation}
and 
\begin{equation}
g\ast\delta'_{x}\in\mathscr{H}_{K_{g}}\Longleftrightarrow\int\left|\xi\right|^{2}\mu\left(d\xi\right)<\infty.\label{eq:g10}
\end{equation}
Note (\ref{eq:g10}) is satisfied for $K_{2}$ but not for $K_{1}$
in \exaref{g3}.
\end{lem}

\begin{proof}
Using $g\ast\delta_{x}=g\left(x-\cdot\right)$, we have $g\ast\delta_{x}\in\mathscr{H}_{K_{g}}$
and 
\[
\left\Vert g\left(x-\cdot\right)\right\Vert _{\mathscr{H}_{K_{g}}}^{2}=\left\langle g\left(x-\cdot\right),g\left(x-\cdot\right)\right\rangle _{\mathscr{H}_{K_{g}}}=g\left(x-x\right)=g\left(0\right)=1.
\]
Equivalently, $\delta_{x}\leftrightarrow\widehat{\delta}_{x}\left(\xi\right)=e^{i\xi x}$,
$\xi\in\mathbb{R}$, and 
\[
\int|\widehat{\delta_{x}}|^{2}d\mu=\int\left|e^{ix\xi}\right|^{2}\mu\left(d\xi\right)=\mu\left(\mathbb{R}\right)=g\left(0\right)=1.
\]

Similarly, $\delta'_{x}\leftrightarrow\widehat{\delta'_{x}}\left(\xi\right)=i\xi e^{ix\xi}$.
Thus $g\ast\delta'_{x}\in\mathscr{H}_{K_{g}}$ if and only if $\int\left|\xi\right|^{2}\mu\left(d\xi\right)<\infty$. 
\end{proof}
\begin{rem}
Given $K_{g}\left(x,y\right)=g\left(x-y\right)$, where $g\left(x\right)=\int e^{i\xi x}\mu\left(d\xi\right)$,
and $\mu$ is a finite positive Borel measure on $\mathbb{R}$, the
reproducing property of $\mathscr{H}_{K_{g}}$ below may be verified
using Fourier-inversion: 
\[
\left\langle g\left(x-\cdot\right),\varphi\ast g\right\rangle _{\mathscr{H}_{K}}=\left(\varphi\ast g\right)\left(x\right),\;\forall x\in\mathbb{R}.
\]
\end{rem}

\begin{proof}
Indeed, we have 
\begin{align*}
\left\langle g\left(x-\cdot\right),\varphi\ast g\right\rangle _{\mathscr{H}_{K}} & =\int\overline{e^{i\xi x}}\widehat{\varphi}\left(\xi\right)\mu\left(d\xi\right)\\
 & =\int_{\mathbb{R}}e^{i\xi x}\widehat{\varphi\ast g}\left(\xi\right)d\xi=\left(\varphi\ast g\right)\left(x\right).
\end{align*}
\end{proof}
\begin{thm}
Fix $K_{g}\in G_{1}$, and let $\mathscr{H}_{K_{g}}$ be the RKHS.
Let $T_{\lambda}:\mathscr{H}_{K_{g}}\rightarrow L^{2}\left(d\lambda\right)$
be as in (\ref{eq:F5}), where $d\lambda$ denotes the Lebesgue measure
on $\mathbb{R}$. Let $f^{g,\lambda}$ be the solution in \lemref{B2},
i.e., 
\begin{equation}
f^{g,\lambda}=T_{\lambda}^{*}\left(\alpha+T_{\lambda}T_{\lambda}^{*}\right)^{-1}\varphi
\end{equation}
where $\varphi\in L^{2}\left(d\lambda\right)$ is fixed. Then, by
the $L^{2}$-Fourier transform, we have 
\begin{equation}
\widehat{T_{\lambda}f^{g,\lambda}}=\left[T_{\lambda}T_{\lambda}^{*}\left(\alpha+T_{\lambda}T_{\lambda}^{*}\right)^{-1}\varphi\right]^{\wedge}=\frac{\widehat{g}}{\alpha+\widehat{g}}\widehat{\varphi}.
\end{equation}

Moreover, the optimal selections from \lemref{c1} and \lemref{d3},
respectively, admit the following explicit spectral representations:
\begin{align}
\left\Vert f^{g,\lambda}\right\Vert _{\mathscr{H}_{K_{g}}}^{2} & =\int\frac{\widehat{g}\left(\lambda\right)}{\left(\alpha+\widehat{g}\left(\lambda\right)\right)^{2}}\left|\widehat{\varphi}\left(\lambda\right)\right|^{2}d\lambda\label{eq:F13}\\
\left\Vert W_{K_{g}}f^{g,\lambda}\right\Vert _{\mathscr{H}_{K_{g}}\times L^{2}\left(d\lambda\right)}^{2} & =\int\frac{\widehat{g}\left(\lambda\right)}{\alpha+\widehat{g}\left(\lambda\right)}\left|\widehat{\varphi}\left(\lambda\right)\right|^{2}d\lambda\label{eq:F14}
\end{align}
\end{thm}

\begin{proof}
By the definition of $T_{\lambda}^{*}$ from (\ref{eq:F6}), it follows
that $T_{\lambda}T_{\lambda}^{*}\varphi=g\ast\varphi\in L^{2}\left(d\lambda\right)$,
and so $\widehat{T_{\lambda}T_{\lambda}^{*}\varphi}=\widehat{g}\widehat{\varphi}$. 

It follows from this, that 
\begin{eqnarray*}
\left\Vert f^{g,\lambda}\right\Vert _{\mathscr{H}_{K_{g}}}^{2} & \underset{\left(\ref{eq:C2}\right)}{=} & \left\Vert \left(T_{\lambda}T_{\lambda}^{*}\right)^{1/2}\left(\alpha+T_{\lambda}T_{\lambda}^{*}\right)^{-1}\varphi\right\Vert _{L^{2}\left(d\lambda\right)}^{2}\\
 & = & \int\frac{\widehat{g}\left(\lambda\right)}{\left(\alpha+\widehat{g}\left(\lambda\right)\right)^{2}}\left|\widehat{\varphi}\left(\lambda\right)\right|^{2}d\lambda
\end{eqnarray*}
and on the other hand, 
\begin{eqnarray*}
\left\Vert W_{K_{g}}f^{g,\lambda}\right\Vert _{\mathscr{H}_{K_{g}}\times L^{2}\left(d\lambda\right)}^{2} & \underset{\left(\ref{eq:d11}\right)}{=} & \left\langle \varphi,T_{\lambda}T_{\lambda}^{*}\left(\alpha+T_{\lambda}T_{\lambda}^{*}\right)^{-1}\varphi\right\rangle _{L^{2}\left(d\lambda\right)}\\
 & = & \int\frac{\widehat{g}\left(\lambda\right)}{\alpha+\widehat{g}\left(\lambda\right)}\left|\widehat{\varphi}\left(\lambda\right)\right|^{2}d\lambda.
\end{eqnarray*}
\end{proof}
\begin{example}
Let $K_{1}$ and $K_{2}$ be the p.d. kernels from (\ref{eq:g8}).
The formulas in (\ref{eq:F13})-(\ref{eq:F14}) take explicit forms,
summarized in \tabref{F1}.
\end{example}

\renewcommand{\arraystretch}{1.5}

\begin{table}[H]
\begin{tabular}{|l|c|c|c|c|}
\hline 
$K_{g}=g\left(x-y\right)$ & $g=\widehat{\mu}$ & $\mu$ & $\frac{\widehat{g}}{\left(\alpha+\widehat{g}\right)^{2}}$ & $\frac{\widehat{g}}{\alpha+\widehat{g}}$\tabularnewline
\hline 
$K_{1}=e^{-\left|x-y\right|}$ & $g_{1}=e^{-\left|x\right|}$ & $d\mu_{1}=\frac{1}{1+\xi^{2}}d\xi$ & $\frac{1+\xi^{2}}{\left(1+\alpha\left(1+\xi^{2}\right)\right)^{2}}$ & $\frac{1}{1+\alpha\left(1+\xi^{2}\right)}$\tabularnewline
\hline 
$K_{2}=e^{-\frac{1}{2}\left(x-y\right)^{2}}$ & $g_{2}=e^{-\frac{1}{2}x^{2}}$ & $d\mu_{2}=e^{-\frac{1}{2}\xi^{2}}d\xi$ & $\frac{e^{\frac{1}{2}x^{2}}}{\left(1+\alpha e^{\frac{1}{2}x^{2}}\right)^{2}}$ & $\frac{1}{1+\alpha e^{\frac{1}{2}\xi^{2}}}$\tabularnewline
\hline 
\end{tabular}

\caption{\label{tab:F1}The p.d. kernels $K_{1}$ and $K_{2}$.}
\end{table}

\renewcommand{\arraystretch}{1}

\bibliographystyle{amsalpha}
\bibliography{ref}

\end{document}